\def\year{2018}\relax
\documentclass[letterpaper]{article} 
\usepackage{aaai18}  
\usepackage{times}  
\usepackage{helvet}  
\usepackage{courier}  
\usepackage{url}  
\usepackage{graphicx}  
\usepackage{amsfonts}
\usepackage{amsthm}
\usepackage{amssymb}
\usepackage{amsmath}
\usepackage{latexsym}
\usepackage{algorithm}
\usepackage{subfigure}
\usepackage{algpseudocode}
\usepackage{enumerate}
\usepackage{listings}
\usepackage{graphicx} 
\usepackage{subfigure}
\usepackage{stfloats}
\usepackage[font=small]{caption}

\newtheorem{prop}{Proposition}
\usepackage{color}

\newcommand{\EU}{\ensuremath{\mathbb{E}U}}

\newcommand{\ab}[1]{\ensuremath{(\alpha_{#1}, \beta_{#1})}}
\newcommand{\abi}{\ab{i}}

\begin{document}
%
\title{Legible Normativity for AI Alignment:  The Value of Silly Rules}

\author{Dylan Hadfield-Menell,\textsuperscript{1,2}
McKane Andrus,\textsuperscript{1,2}
Gillian K. Hadfield,\textsuperscript{2,3,4,5}\\
\textsuperscript{1}{Department of Electrical Engineering and Computer Science, University of California Berkeley}\\
\textsuperscript{2}{Center for Human Compatible AI}\\
\textsuperscript{3}Faculty of Law and Rotman School of Management, University of Toronto\\
\textsuperscript{4}{Vector Institute for Artificial Intelligence}\\
\textsuperscript{5}{OpenAI}\\
dhm@eecs.berkeley.edu, mckaneandrus@berkeley.edu, g.hadfield@utoronto.ca}

\maketitle

\begin{abstract}
It has become commonplace to assert that
autonomous agents will have to be built to follow human rules of
behavior--social norms and laws. But human laws
and norms are complex and culturally varied systems; in many cases agents will have to \textit{learn}
the rules. This requires autonomous agents to have models of how human rule systems work so that they can make reliable predictions about rules. In this paper we contribute to the building of such models by analyzing an overlooked distinction between important rules and what we call silly rules \textemdash rules with no discernible direct impact on welfare. We show that silly rules render a normative system both more robust and more adaptable in response to shocks to perceived stability. They make normativity more legible for humans, and can increase legibility for AI systems as well. For AI systems to integrate into human normative systems, we suggest, it may be important for them to have models that include representations of silly rules. 
\end{abstract}

\noindent 

\section{Introduction}
As attention to the challenge of aligning artificial intelligence with human welfare has grown, it has become commonplace to assert that autonomous agents will have to be built to follow human norms and laws \cite{etzioni2016ethics,etzioni2017,IEEE2018}. But this is no easy task. Human groups are thick with rules and norms about behavior, many of which are largely invisible, taken for granted as simply ``the way things are done" by participants \cite{Schutz1964}. They are constituted in complex ways through second-order \emph{normative beliefs}: beliefs about what others believe we should or should not do in some situation \cite{bicchieri2006grammar,bicchieri2017}. Human laws and norms are frequently ambiguous and complicated, they vary widely across jurisdictions, cultures, and groups, they change and adapt. The cases in which they are reducible to formal rule statements is the exception. Even deciding whether a vehicle has violated a numerical speed limit is far from straightforward:  was visibility poor? were there children nearby? Adding to the complexity, rules and norms are enforced both by formal institutions like courts and regulators through costly and error-prone procedures and by the informal behavior of agents through third-party criticism and exclusion or internalization and self-criticism.  This means that what actually counts as a rule can easily diverge from announced or formal rules and that rule-based environments are complex dynamic systems. As a result, we cannot rely on formal rules simply being imposed on agents \textit{a priori}; instead, agents will in many cases have to \textit{learn} the rules and how they work in practice. Normativity--the human practice of classifying some actions as sanctionable and others as not and then punishing people who engage in sanctionable conduct--will have to be \emph{legible} \cite{Dragan2013} to AI systems. 

In this paper, we introduce a distinction between types of rules that can aid in building predictive models to make human normative systems legible to an AI system.  We distinguish between \emph{important} rules and \emph{silly} rules. An important rule is one the observance of which by one agent generates direct payoffs for some other agent(s).  When an agent complies with rules prohibiting speeding, for example, other agents enjoy a material payoff as a direct consequence, such as a reduced probability of accident.  A silly rule, in contrast, in one the observance of which by one agent does not generate any direct material payoff for any other agent. When an agent violates a dress code, for example, such as by failing to wear a head covering in public, no-one is materially affected as a direct consequence of the violation.  Observers might well be offended, and they might punish the violator, but the violation itself is inconsequential.

We ground our claim that the distinction between silly and important rules will prove important to building models for aligned AI using Monte Carlo simulations. We show that silly rules promote robustness and adaptation in groups. Silly rules perform a legibility function for humans--making it easier for them to read the state of the equilibrium in their group when equilibrium is threatened.  Incorporating this insight about silly rules into AI design should allow human normative systems to be more legible to AI.  

Our paper is presented as follows.  We first illustrate the concept of silly rules an example drawn from a concrete environment.  We then develop a model of groups, based on \cite{hadfield2012law}, in which a group of agents announces a set of rules and relies exclusively on voluntary third-party punishment by group members to police violations. We first show formally that, if silly rules are costless, groups with more silly rules achieve higher payoffs. We then consider the case in which following and punishing silly rules is costly and present the results of our simulations.  Our results demonstrate that groups with lots of (sufficiently cheap) silly rules are more robust:  they are able to maintain more of their population and are less likely to collapse than groups with fewer silly rules in response to an unfounded shock to beliefs about the proportion of punishers.  Groups with lots of silly rules are also more adaptable: they collapse more quickly when there is a true drop in the proportion of punishers below the threshold that makes group membership valuable.  

Our contributions are threefold.  First, we present a formal model that can account for the presence of silly rules in a normative system and show the conditions under which silly rules are likely to exist. This is a contribution to normative theory in human groups. Second, this work provides an example of the importance of building predictive models of human normative systems \emph{qua} systems--not merely predicting the presence of particular norms, which is the dominant approach taken in the growing literature on AI ethics and alignment. Third, we demonstrate that standard AI methods can be valuable tools in analyzing human normative systems.

\section{What are Silly Rules? A Thought Experiment from Ethnography}
One of the challenges of building models of human normativity is that as researchers we are all participants in our taken-for-granted normative environments and this can make it hard to study norms scientifically  \cite{Haidt2008}. To attempt to overcome this, we motivate our work with an example drawn from an ethnography of a group that engages in practices far removed from the worlds in which AI researchers live.  Moreover, we will use the shocking label "silly rules" in order to illuminate an overlooked distinction in the context of the existing literature on normativity. Most of the social science of norms focuses on functional accounts of particular norms such as norms of reciprocity, fair sharing of rewards, or non-interference with property. These accounts argue that particular norms evolve because they support human cooperation and thus improve fitness \cite{boyd2009culture,tomasello2013review} or solve coordination games \cite{Sugden1986,McAdams:2005,Myerson:2004,McAdams2015} for example. Our work highlights the systemic functionality of rules that, individually, have no direct functionality. All human societies, we will show, are likely develop silly rules, for good functional reasons.  

Suppose that an AI system were tasked with learning how to make arrows by observing the Aw\'a people of Brazil. The Aw\'a are hunter-gatherers now living in relatively small numbers on reservations established by the Brazilian government. One of the things the AI will observe, like the ethnographers who have studied these people, is that the men of the Aw\'a spend four or more hours a day making and repairing arrows \cite{Gonzalez2011}. They are produced in large quantities and need frequent repair. They are between 1.4 and 1.7 meters in length, customized to the height of their owner. Bamboo collected to make the points is sometimes shared but the arrows themselves are not; they are buried with their owner. The men use only dark, not brightly colored feathers. All parts of the arrow \textemdash shaft, point, and feathers \textemdash are smoked over a grill during preparation and the arrows themselves are kept warm in smoke at all times unless they are bundled and put in storage in the rafters of a hut. 
 Will the AI system reproduce all of these arrow-making behaviors? We can imagine that AI designed with principles of inverse reinforcement learning \cite{ng2000algorithms} might discern which behaviors actually contribute to the functionality of the arrows--which is presumably what the human designer intended \cite{Hadfield-Menell:2017b}. According to the human ethnographers who observed the Aw\'a, many of the arrow-making practices are not functional. Even if smoking the wood used in the shaft of the arrow during manufacture contributes to a harder, straighter arrow, smoking the feathers seems unnecessary, as does ensuring the arrows are kept warm at all times. Moreover, the men make and carry many more arrows than they will use. In one season, a total of 402 arrows were carried on 9 different foraging trips; 9 were used. Most game on these trips was shot with a shotgun \cite{Gonzalez2011}.
 
 An AI system that ignored the non-functional arrow-making behaviors, however, would be violating the norms of the Aw\'a people. The arrow-making practices described above are not just practices; they are rules.  They reflect normative expectations \cite{bicchieri2006grammar}. How do we know? The lack of functionality is one clue: the Aw\'a presumably have also discovered that a cold arrow works and that they spend a lot of time making arrows that go unused and are damaged by being bundled and carried around. But the better evidence comes from how they respond to the only man who makes his arrows differently.  This man is mocked: his arrows are exceedingly long (2.3 m) and he uses brightly colored feathers. He is ``the only man who does not socialize with the rest of the village." His strange arrows ``are another sign of his loss of `Aw\'a-ness'" \cite{Gonzalez2011}. 
 The Aw\'a's rules are normative, moral principles: bright colored feathers are used only by women to prepare headbands and bracelets used by men in religious rituals and are associated with the world of spirits and ancestors; the making of fire and cooking are associated with masculinity and divinity. An AI system that violated these rules in the pursuit of arrow-making would not be aligned with the moral code of the Aw\'a.

We call the non-functional rules "silly rules".  We emphasize that silly rules are not ``silly" to the groups that follow them. They can have considerable meaning, as they do to the Aw\'a.  Our results will show why silly rules can be very important to the overall welfare of a group and hence the subject of intense concern by group members.

\section{Model}

Our model is based on a framework developed in \cite{hadfield2012law}.  We characterize a set of agents as a group defined by a fixed and common knowledge set of rules. A rule is a binary classification of alternative actions that can be taken in carrying out some behavior. Actions are either ``rule violations" or ``not rule violations". 

Members of this group engage in an infinite sequence of interactions, each of which is governed by a rule drawn randomly from the ruleset. Each interaction is composed of a randomly selected pair of agents and a third actor, whom we will call a scofflaw, who will choose either to comply with the governing rule or not.  (For tractability reasons, we do not model the scofflaws as group members.) One of the agents is randomly designated as the victim of the rule violation; the other is a bystander. If a rule is an \emph{important} rule, the victim incurs a benefit if the rule is enforced and incurs a cost if not. If a rule is a \emph{silly} rule, the victim incurs no benefit from the scofflaw's compliance with the rule and no cost from violation.

Group members are of two types in the bystander role: punishers, who always punish a rule violation, and non-punishers, who never punish. We assume that groups members signal whether they are punishers by paying a signaling cost in each interaction before a potential violation occurs. (\cite{boyd2010coordinated} show that signaling punisher status supports an evolutionarily stable equilibrium in which non-punishers cannot free-ride on punisher types.)   The scofflaw complies with the selected rule if the bystander is a punisher and violates it otherwise. There is no punishment in equilibrium, but the model can be seen as assuming that victims always punish but punishment is only effective when bystanders punish as well.

Prior to each interaction, group members have an option to quit the group and take a risk-free payoff. We formalize this setup as follows.

Each interaction is a game $g$ and we define the sequence for a group as a tuple: $\langle G, T_{\theta}, \Pi, U,
\gamma, c\rangle$ where $G$ is a distribution over games and
$T_{\theta}$ is a distribution over punishment types $t$ in the group, where $t=1$ if an agent is a punisher and $t=0$ if not. The proportion of punishers is given by $\theta \in{[0,1]}$. 
For the tractability of our agent models, we treat $T_\theta$ as a static distribution, and assume agents do likewise, even though it is subject to change as individuals leave the group. 

We will abuse notation somewhat and use $T$ and
$G$ to refer to the support of the corresponding distributions where
the meaning is obvious. $\Pi$ is each agent's prior distribution over the
parameters of $T_{\theta}$, and
$U~:~G~\times~T_{\theta}\rightarrow~\mathbb{R}$ is a mapping from
types and games to immediate payoffs for the agents. $\gamma$ is each agent's discount parameter for
future rewards. $c$ expresses a participation cost. This can be
understood as the expected cost of an agent in the bystander role to signal
that she is a punisher to the other agent in an interaction.\

Every agent begins in period 1 with perfect knowledge
of how actions are classified, all payoffs, and the distribution
of games. The agents do not know the distribution of types in the group, but they do hold a prior which we will specify shortly. The agents update their beliefs about the distribution of types using Bayes' rule. The super game is
defined as follows:
\\ For each period $j$:
\begin{enumerate}
  \item Each agent chooses whether to participate or not. If an agent opts out,
     she collects $0$ payoff. 
  \item All agents that opt in are matched with another agent at random. A game $g_{j+1} \sim G$ is drawn for each of the agent pairings. 
  \item Punishers incur a cost $c$ to signal that they will punish violations. All players observe these signals. 
    \item In each pairing one agent is randomly assigned the role of victim, $V$, and the other the bystander, $B$. All players observe the result of this random assignment.
  \item All players learn whether the game is a silly or important game.
  \item If $B$ is a punisher, the scofflaw complies with the rule. Otherwise, the scofflaw violates the rule.
  \item Victims and bystanders collect payoffs given by $U_V(g_j, t_B, t_V)$ and $U_B(g_j,t_V,t_B)$.
\end{enumerate}
Agents that play in the bystander role in any game incur no benefit; they incur the cost $c$ if they are a punisher and $0$ if not. Agents that play in the victim role receive a payoff of $0$ in games governed by a silly rule.  In games governed by an important rule they receive a positive reward $R$, if $B$ is a punisher and a negative reward, $-R$, if $B$ is not.\par
We formalize the set of important games as follows:
\begin{align*}
G' = \{g \in G| U(g, \cdot) \ne 0\}\\
U_V(g, t_O, t_V|g\in G') =(2t_O-1)R-(2t_V-1)c \\
U_O(g,t_O, t_V|g\in G') = (2t_O-1)c
\end{align*}
We will use $\EU = \mathbb{E}_{g,t_O}[U_V(g, t_O,t_V)|g\in G']$ to denote the
expected utility of an important game. We let $d$ denote the
\emph{density} of the process generating games: the probability of a silly game. $$d = 1 - P(g \in G'); g\sim G.$$ Note that a super-game has high density ($d$ close to 1) when silly rules are a large fraction of the ruleset.
\par

Critically, we ensure that the density of silly games does not alter the (expected) rate at which important games are presented to the agents. Rather than take the place of important interactions, in our model silly interactions increase the \emph{total number} of interactions happening in the same time frame. To be
concrete,  we assume the expected discounted reward obtained from
important games is independent of $d$. This condition can be attained
through a suitable modification of $\gamma$ as a function of $d$:
\begin{prop}
Setting $$\gamma_d = 1-(1-d)(1-\gamma)$$ ensures that the expected sum
of discounted rewards from important games is independent of $d$\footnote{$\mathbb{I}[\psi]$ is the indicator function for
    the condition $\psi$.}.
\end{prop}
\begin{proof} See appendix. \end{proof}

It can be easily shown that this constrained model describes an \emph{optimal stopping problem}~\cite{gittins2011multi}. Each agent in our model must choose between participating, in which case they get an unknown reward and learn about the enforcement equilibrium in the community, or opting out, in which case they stop participating get a constant reward of 0. A classic result from the literature on optimal stopping tells us that, in the optimal policy, if the agent opts out once it will opt out for the rest of time. This is because the agent's information state does not change when it chooses to opt out, so if it was optimal to stop at time $t-1$, it will also be optimal to stop at time $t$. Thus, we refer to the decision not to participate at any point, then, as a decision to retire.
 \begin{figure*}
    \centering
    \centerline{\includegraphics[width=.8\textwidth]{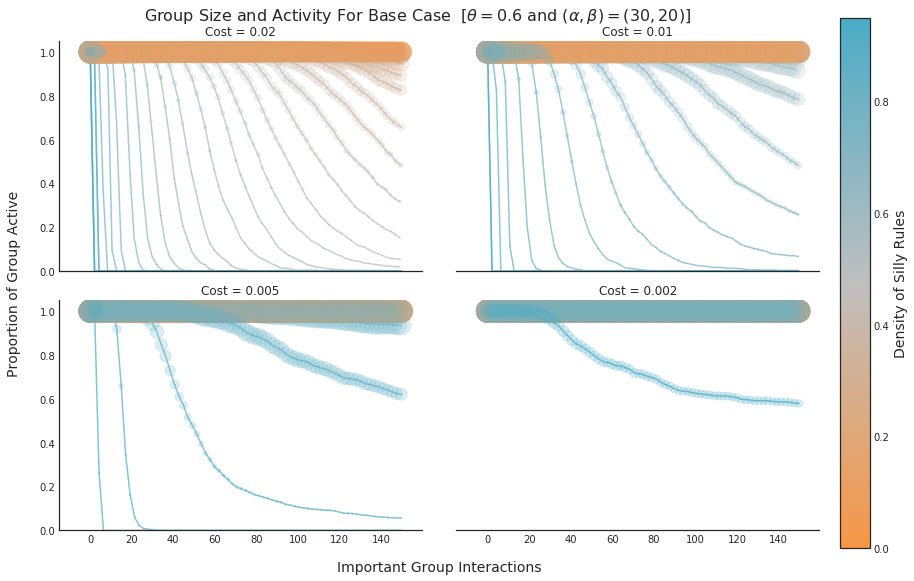}}
    \caption{The y-axis represents the proportion of the 1000 groups with at least 2 individuals left, where the x-axis represents time in terms of the number of expected important interactions per agent. The size of the bubbles signifies the average size of the 1000 groups at the given point in time. 40 linearly spaced values between $0.0$ and $0.95$ were used for silly rule density, and the graphs for each setting are colored accordingly. For cost 0.02, we see that high density (blue) groups collapse rapidly, while the lowest density groups (orange) sustain their size for the duration of the experiment. As the cost comes down, higher density groups start to survive.}
    \label{fig:Hypo1}
\end{figure*}

These problems broadly fall under the class of \emph{partially observed Markov decision processes} (POMDP)~\cite{sondik1971optimal}. In a POMDP the optimal policy only depends on the agent's \emph{belief state}: the agent's posterior distribution over the hidden state of the system. In this case, this is a distribution over the enforcement likelihood in the community. We give our agent a beta prior over this parameter so that the belief space for our agent is a two dimensional lattice equivalent to $\mathbb{Z}_+^2$. Initially, the belief state is $(\alpha_0, \beta_0)$.  The probability that the bystander is a punisher in the first game is $$p_{\alpha\beta} =
\frac{\alpha}{\alpha+\beta}.$$ Once the games begin, agents update their prior beliefs using Bayes' rule, adding the counts of punishers and non-punishers observed to the prior values. In the following, we will use $\alpha_i (\beta_i)$ to represent the number of punishers (non-punishers) observed prior to round $i$.


\section{Theoretical Analysis: The Value of Dense Normative Structure}

Consider first the case in which the signaling cost, $c$, is zero. In this case, punishers only face a risky choice when they are assigned to the victim role in an important game. In all other periods, the per-period expected payoff of playing the risky arm is a constant 0. 

Intuitively, the benefit of higher density of unimportant games is that the agent is in a more information rich environment. In general, this benefit trades off with the cost of signalling. However, when the signalling cost is 0, a higher density is strictly better. One way to show this is to consider the \emph{value of perfect information} (VPI): the additional utility an agent can get in expectation when it has full information compared with the expected utility with partial information ~\cite{Russell+Norvig:2010}.  We can show that, in the limit as density goes to 1, VPI goes to 0; high density of unimportant games essentially removes the agents' uncertainty over the proportion of punishers. 

\begin{prop} If the participation cost, $c$, is 0, then, for any belief state, \abi, and discount rate $\gamma$, the corresponding VPI goes to zero as density goes to 1. That is
\begin{equation}
\lim_{d\rightarrow 1}VPI(\abi; d, \gamma) = 0
\label{eq:vpi-limit}
\end{equation}
\label{prop:vpi}
\end{prop}

\begin{proof} (Sketch; see Appendix for details.) We consider a policy such that the agent participates in order to observe $\tau(d)$ interactions. After $\tau(d)$ observations, it uses its best estimate of the probability of enforcement to decide if it should leave. It doesn't reconsider retiring or rejoining afterwards. We show that this stopping time function can be chosen so that the expected number of important games goes to 0, so it doesn't lose utility in expectation, while the total number of interactions (including silly games) goes to $\infty,$ so it makes the retirement decision with perfect information in the limit. 
\end{proof}

It is straightforward to show that VPI is strictly positive as long as $P(V(\theta) >0) > \epsilon > 0$ and $P(V(\theta) < 0) > \epsilon > 0$ for some finite epsilon. Combined with our proposition, this means that environments with more silly rules will be higher value to agents; as the density of silly rules goes to 1, we can neglect the utility lost due to partial information about the proportion of punishers. Where participation costs can be neglected, an agent will prefer an environment with lots of silly rules.

\section{Monte Carlo Experiments}
\begin{figure}
    \includegraphics[width=3in]{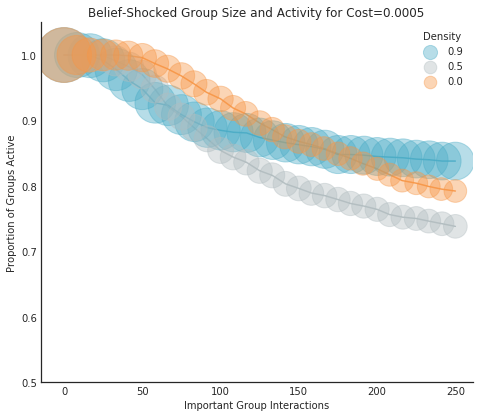}
    \caption{Using the same graphical representation as Figure 1, here we see a comparison of the robustness between groups of groups at 3 different density values. As the number of interactions increases, the survival rate and average population size of groups with higher silly rule density surpasses those of lower density groups.}
    \label{fig:Hypo2}
\end{figure}

To test the benefit of silly rules in groups composed of our previously defined agents, we constructed a series of simulation-based experiments in which we manipulated the density of silly rules, cost of signaling, distribution of punishers, and prior beliefs about the punisher distribution. Each simulation was carried out in a group of 100 agents, each given the type of punisher or non-punisher. The simulations were broken down into discrete periods, or \textit{group interactions}, in which each individual was matched with another and engaged in an interaction, silly or important. We set the reward for the victim in an important game to $+1$ in the case in which the bystander is a punisher, and $-1$ if the bystander is not. Note that given the symmetry in gains and losses in important games, continued participation in the group is valuable if the likelihood that a bystander in an important game is a punisher is greater than $.5$ Given the density-adjusted discount factor for each simulation, the expected reward of 10 interactions in a $d=0.9$ environment would be equivalent to that of 1 interaction in a $d=0.0$ environment. This allows us to normalize the periods into \textit{timesteps}, where 1 \textit{timestep} is equal to $\frac{1}{1-d}$ \textit{group interactions}. 

\begin{figure}
    \includegraphics[width=3.3in]{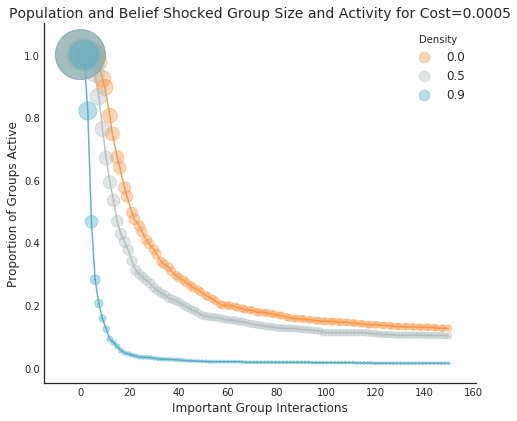}
    \caption{Using the same graphical representation as Figure 1 and 2, here we see a comparison of the adaptability between groups of groups at 3 different density values. Given the unstable conditions of the scenario, higher adaptability corresponds to faster a collapse rate. Groups with high density are seen to collapse much faster than those with low density.}
    \label{fig:Hypo3}
\end{figure}

We are interested in the size of groups over time, as agents make decisions about whether to remain in the group or not given their interaction experience, as a function of the density of silly rules. 

To establish a base case simulation, we first consider the case in which there is low uncertainty about the proportion of punishers in the group.  We can think of this as the case in which a stable group has engaged in interactions over a long period of time and all agents have many observations of the proportion of punishers in the group. We set $\theta$, the punisher proportion, to be 0.6 and the alpha-beta prior of the agents to $30:20$, which implies high confidence in the agents' estimate of $\theta$. Note that with this ground truth, group membership is valuable, generating an expected payoff higher than the alternative of $0$. We then run the simulation for $1000$ groups on the full factorial of $20$ logistically scaled signaling costs, $c$, and $20$ linearly scaled densities, $d$. Doing so confirms our first hypothesis (see Figure 1):

\textit{Hypothesis 1: When uncertainty about the proportion of punishers in a group is low, the likelihood that a group loses members and the likelihood the group collapses increases with the density of silly rules.}

In this case, as the frequency of the cost of signaling that one is a punisher in silly interactions increases it begins to outweigh the possible rewards from important interactions. We see that as cost goes up, groups with higher density of silly interactions shrink and collapse more frequently than those with low density.  This confirms the intuition that silly rules are costly if they serve no information function.

\subsection{Group Robustness}
Having established a baseline, we investigate the benefits of silly rules for a group by considering different scenarios that will stress test the robustness and adaptability of a group. The first case we consider is one in which the individuals' beliefs in a stable group are shocked, lowering their confidence in the proportion of punishers. Concretely, this involves setting the beta priors to $1.2:0.8$ instead of $30:20$. Our hypothesis for the belief-shock scenario is as follows:

\textit{Hypothesis 2: For sufficiently low signaling cost, a higher density of silly rules increases a group's resilience to shocks in individuals' beliefs about the distribution of punishers.}

As shown in Figure 2, in settings with low signaling cost, high density allows for quick stabilization and strong individual retention. Around $75\%$ of groups with $0.9$ density persist after $250$ \textit{timesteps}, with an average population of $~50$. Compare this to the lower density groups, where the groups that survive lose most members before stabilizing.

\subsection{Group Adaptability}
To test adaptability, we imagine an alternative scenario in which the shock to individuals' beliefs is accompanied by a change in the ground truth about the proportion of punishers. Concretely, we change the beta prior to $1.2:0.8$ once again, and set $\theta$ to be $0.4$. With fewer punishers than non-punishers, participation in the group generates a negative expected payoff, and agents would do better to leave the group.  Put differently, the group's ruleset is no longer generating value for group members.  In a negative-value group such as this, we define the adaptability of the group to be the rapidity of collapse. Our hypothesis for this scenario is as follows:

\textit{Hypothesis 3: For sufficiently low signaling cost, a higher density of silly rules allows for faster adaptation to negativer shocks in the distribution of punishers in a group.}

Looking to Figure 3, we find support for this hypothesis in the experiments. After only a few \textit{timesteps}, we see that the high density groups are mostly collapsed, whereas the lower density groups take quite a bit longer to peter out.

\section{Discussion}  

Our results show that silly rules help groups adapt to uncertainty about the stability of social order by enriching the information environment.  They help participants in these groups track their beliefs about the likelihood that violations of important rules will be punished, and thus the likelihood that important rules will be violated.  These beliefs are critical to the incentive to invest resources in interaction.   

We focus on the punisher type of bystanders because third-party punishment is the distinctive feature of human groups \cite{riedl2012chimps,tomasello2013review,buckholtz2012roots}; it extends the range of actions that can be deterred from those deterred by the reactions of the victim alone to those that can be deterred by group punishment \cite{boyd1992punishment}.

What are the lessons for AI alignment research? The goal of AI alignment is the goal of building AI systems that act in ways consistent with human values.  For groups of humans, this means (at least) values reflected in rules of behavior. Discerning values from rules is complex: some rules reflect important values, such as honoring a promise or avoiding harm.  Others do not reflect values that are important \emph{per se}. For an AI system to make good inferences and predictions from observing normative behavior, then, it will need to distinguish between important rules and  silly rules.

Failing to make this distinction could lead to at least two key inferential errors. One error would be to treat important and silly rules as equally likely to vary over time and place. But important rules, because they promote functionality in human interactions, are likely to vary only when there is some causal reason.  
Silly rules, on the other hand, can vary for any reason, or none.  
Modelling the distinction between silly and important rules is essential to accurately learning rule systems.  An AI system that lacks this distinction will over-estimate the likelihood of encountering certain types of normative behavior--with respect to dress codes, for example--while under-estimating the likelihood of others, such as speeding rules. 

A second error that could result from a failure to distinguish between important and silly rules is that an AI system is likely to treat all rules that it sees enforced as equally important to human values. This would produce a good solution in ordinary circumstances. But this will produce a poor solution in circumstances in which it would be very costly to comply with all the rules. If an AI system treats all of rules as equally important to humans, it will presumably economize equally across the rules.  But the better solution is to prioritize important rules and compromise on silly rules. 

The distinction between silly and important rules also raises a question for work on human-robot interaction:  how important is it for an AI system to help enforce silly rules? Our model brings out a legibility function in silly rules--they make it easier for agents in a group that depends on third-party enforcement to discern the stability of the rules in light of uncertainty generated by changes in population or the environment.  If artificial agents are interacting in these environments and they don't participate in enforcing silly rules, what impact does that have on the beliefs of human agents? Does the introduction of large numbers of artificial agents who ignore silly rules into a human group (such as self-driving cars into the group of humans driving on highways) have the same impact on the robustness and adaptability of the group as a decrease in the density of silly rules, by reducing the amount of information gained from the opportunity to observe bystander behavior in interactions? Further still, when a robot learns and enforces silly rules, do these seemingly arbitrary norms become reified, fundamentally changing their meaning and reducing their signaling potential? We leave these questions, and others, for future research.


\bibliographystyle{aaai}

\newpage
\section*{Appendix: Proofs}
\setcounter{prop}{0}
\begin{prop}
Setting $$\gamma_d = 1-(1-d)(1-\gamma)$$ ensures that the expected sum
of discounted rewards from important games is independent of $d$:
\begin{align*}
    \forall d, \in [0, 1)\ \ \mathbb{E}_{g_j,
    t_O}\left[\left.\sum_{j=0}^\infty \gamma^j U_V(g_j,t_O,t_V)\right|
    g_j\in G'\right] = \\
    \mathbb{E}_{g_j,t_O}\left[\left. \sum_{j=0}^\infty \mathbb{I}[g_j\in
     G']\gamma_{d}^j U(g_j,,t_O,t_V)\right | d
   \right].
\end{align*}
\end{prop}
\begin{proof} 
We first show that it is sufficient to ensure that the expected value
of $\gamma_d^{j}$ is the same given that $j$ is a round with an
important game:
\begin{align*}
&\mathbb{E}_{g_j, t_O}\left[\left. \sum_{j=0}^\infty \mathbb{I}[g_j \in G'] \gamma_{d}^j U(g_j,t_O,t_V)\right | d \right] \\ 
&= \sum_{j=0}^\infty \mathbb{E}_{g_j, t_O}\left[\left. \mathbb{I}[g_j\in G'] \gamma_{d}^j
  U(g_j,t_O,t_V)\right | d \right] \\ 
&= \sum_{j=0}^\infty \gamma_d^j
\mathbb{E}_{g_j, t_O}[U(g_j, t_O, t_V) | d, g_j \in G']\mathbb{E}_{g_j}\left[\mathbb{I}[g_j\in G'] | d \right]  \\
&=(1-d)\EU \sum_{j=0}^\infty \gamma_{d}^j
\end{align*}
where the first line holds by the linearity of expectation, the
fact that $g_j$ is an independent iid draw from a stationary
distribution, and the constraint on the agent's beliefs that $t_O$ is a also an independent iid draw from a stationary distribution. Substituting the form of the infinite geometric series,
we see that
\begin{equation}
\frac{\EU}{1-\gamma} = \frac{(1-d)\EU}{1-\gamma_d}
\end{equation}
is sufficient to achieve our goal. Substituting the form for
$\gamma_d$ in the theorem statement and reducing shows that this
condition is satisfied.
\end{proof}

\begin{prop} If the participation cost, $c$, is 0, then, for any belief state, \abi, and discount rate $\gamma$, the corresponding VPI goes to zero as density goes to 1. That is
\begin{equation}
\lim_{d\rightarrow 1}VPI(\abi; d, \gamma) = 0
\label{eq:vpi-limit}
\end{equation}
\label{prop:vpi}
\end{prop}

\begin{proof}
Let $V(\theta)$ be the expected value of participating forever, given $\theta$. 
The optimal full information policy will retire whenever
$V(\theta) < 0$ and has value $V_+(\theta) = \max\{V(\theta), 0\}$. VPI is
the difference between the expected value of $V_+$ and the value of
the optimal partial information policy $V(\abi; d, \gamma)$:
\begin{gather}
\nonumber VPI(\abi; d, \gamma) = \\  
 \mathbb{E}\left[\left.V_+(\theta) \right | \abi\right] - V(\abi; d, \gamma)
\label{eq:vpi}
\end{gather}

We proceed by lower bounding $V$. $V$ is the value of the optimal
policy so it is weakly lower bounded by any arbitrary policy. We consider a policy that participates for 
\begin{equation}
    \tau(d) = \frac{1}{\sqrt{1-d}}\label{eq:tau}
\end{equation}
rounds and then retires if $\mathbb{E}[V(\theta)] < 0$. This choice of $\tau$ ensures that
\begin{align}
    \lim_{d\rightarrow 1} \tau(d) &= \infty; \label{eq:tau-infty}\\
    \begin{split}
    \lim_{d\rightarrow 1} \sum_{t<\tau(d)} P(g_t \in G) &= \lim_{d\rightarrow 1}\frac{1-d}{\sqrt{1-d}} \\
    &= 0 .\label{eq:no-important}
    \end{split}
\end{align}

\eqref{eq:tau-infty} ensures that, as density goes to 1, then agent's estimate of participation value when it decides, $\mathbb{E}[V(\theta)]$ converges to $V(\theta)$ by consistency. \eqref{eq:no-important} ensures that the expected number of important games (and thus opportunities to lose utility against a full information agent) goes to 0. This is sufficient to show that 
\begin{equation}
\lim_{d\rightarrow 1} V(\abi; d, \gamma) = \mathbb{E}[V_+(\theta)|\abi ]
\end{equation}
which shows the result. 
\end{proof}

\end{document}